\documentclass[a4paper]{article}
\usepackage[a4paper, margin=1in]{geometry}
\usepackage{indentfirst}
\usepackage{amssymb}
\usepackage{mathtools}
\usepackage{setspace}
\usepackage{color, soul}
\usepackage{listings}
\usepackage{todonotes}
\usepackage{multirow}
\usepackage{epstopdf}
\usepackage{caption, subcaption}
\usepackage{hyperref}

\definecolor{dark-gray}{gray}{0.3}

\newtheorem{lemma}{Lemma}
\newtheorem{remark}{Remark}
\newenvironment{proof}{\paragraph{Proof:}}{\hfill$\square$}

\title{\textbf{A unified framework for Hamiltonian deep neural networks}}

\author{%
{Clara Luc\'ia Galimberti$^1$}, {Liang Xu$^1$} and {Giancarlo {Ferrari Trecate}$^1$} \\
\\
\texttt{\{clara.galimberti, liang.xu, giancarlo.ferraritrecate\}@epfl.ch}\\
\\
$^1$Institute of Mechanical Engineering, \\École Polytechnique Fédérale de Lausanne, Switzerland.%
}
\begin{document}

\maketitle

\begin{abstract}%
	Training deep neural networks (DNNs) can be difficult due to the occurrence of vanishing/exploding gradients during weight optimization.
	To avoid this problem, we propose a class of DNNs stemming from the time discretization of Hamiltonian systems.
	The time-invariant version of the corresponding Hamiltonian models enjoys marginal stability, a property that, as shown in previous works and for specific DNNs architectures, can mitigate convergence to zero or divergence of gradients.
	In the present paper, we formally study this feature by deriving and analysing the backward gradient dynamics in continuous time.
	The proposed Hamiltonian framework, besides encompassing existing networks inspired by marginally stable ODEs, allows one to derive new and more expressive architectures.
	The good performance of the novel DNNs is demonstrated on benchmark classification problems, including digit recognition using the MNIST dataset.
\end{abstract}

\section{Introduction}
Deep learning has achieved remarkable success in different fields like computer vision, speech recognition and natural language processing~\cite{He2016, XiongNLP}.
There is also a growing interest in using deep neural networks (DNNs) for approximating complex controllers~\cite{Lucia2018ifac, Zoppoli2020book}.
In spite of several progresses, the training of DNNs still presents some difficulties.
Most optimization algorithms for DNNs, such as stochastic gradient descent, involve the computation of gradients that, as observed in \cite{Bengio94},  can explode or vanish, hence making the learning problem ill-posed.

Recently, it has been shown that these issues are mitigated for specific classes of DNNs which stem from the time discretization of ordinary differential equations (ODEs)~\cite{Haber_2017, haber2017, Chang19, lu2017finite, Weinan2017}.
These results leverage the stability properties of the underlying ODE for characterizing relevant behaviours of the corresponding DNNs~\cite{Haber_2017}.
Specifically, instability of the underlying ODE results in unstable forward propagation in the DNN model, while convergence to zero of system states can lead to vanishing gradients during training.
This observation suggests using DNN architectures based on dynamical system models that produce bounded and non-vanishing state trajectories.
An example is provided by ODEs based on skew-symmetric maps, which have been used in~\cite{Haber_2017, Chang19} for defining anti-symmetric DNNs.
Another example is given by dynamical systems in the form
\begin{align}\label{eq:simpleHamiltonianSystems}
\dot{\bf y}=-\nabla_{\bf z} H({\bf y},{\bf z}), \quad \dot{\bf z}=\nabla_{\bf y} H({\bf y},{\bf z}),
\end{align}
where $H(\cdot, \cdot)$ is a Hamiltonian function.
This class of ODEs has motivated the development of Hamiltonian-inspired DNNs in~\cite{Haber_2017}, whose effectiveness has been shown in several benchmark classification problems \cite{Haber_2017, Chang19, Chang18a}.

However, all these works consider only restricted classes of skew-symmetric maps or particular Hamiltonian functions, which, together with the specific structure of the dynamics in~\eqref{eq:simpleHamiltonianSystems}, limit the representation power of the resulting DNNs.

In this work, we leverage general models of Hamiltonian systems~\cite{vanderSchaft2017} and provide a unified framework for defining Hamiltonian DNNs (H-DNNs for short), which, under very mild assumptions, encompasses anti-symmetric and Hamiltonian-inspired networks.
Furthermore, we define new classes of DNNs, which are more expressive than those in~\cite{Haber_2017} and \cite{Chang18a,Chang19}, and can achieve comparable performance while using less layers.
We show this feature using several benchmark classification problems, including digit recognition based on the MNIST dataset.

Hamiltonian dynamics can be marginally stable by construction, independent of the specific model parameters.
We leverage this property and the use of regularized loss functions penalizing the variation of weights over consecutive layers~\cite {Haber_2017} for studying the well-posedness of the training process.
To this purpose, we first consider the simplified setting where DNN weights are constant across layers, which corresponds to letting the regularization parameter grow to infinity.
By analysing the underlying ODE, we prove the marginal stability of the backward gradient dynamics, which implies the absence of vanishing/exploding gradients during training.
In addition, we present a simulation study showing that this property is also verified when the regularisation parameter is finite and network parameters can change across layers.

The reminder of our paper is organized as follows.
Related works are discussed in Section~\ref{sec:related_work}. 
In Section~\ref{sec:methods}, we present H-DNNs, and analyse the stability properties of the backward gradient dynamics.
Numerical examples are provided in Section~\ref{sec:examples}, which is followed by concluding remarks in Section~\ref{sec:conclusion}.
{Throughout this work, we use the column convention for gradients, i.e. the gradient $\nabla{f}$ of a real-valued function $f({\bf x})$ is a column vector.
	
\section{Related works}\label{sec:related_work}

\subsection{DNN induced by ODE discretization} \label{sec:ODEinspirednet}

The connection between neural networks and ODEs can be established by starting from the nonlinear system
\begin{equation}
\dot{\bf y}(t) = {\bf f}({\bf y}(t), {\boldsymbol{\theta}}(t)),\quad {\bf y}(0) = {\bf y}_0, \quad {\bf y}(t)\in \mathbb{R}^n, \quad 0 \leq t \leq T,
\label{eq:firstorderODE}
\end{equation}
where ${\boldsymbol{\theta}}(t)\in \mathbb{R}^{n_\theta}$ is a vector of parameters.
For a given $N\in \mathbb{N}$, we consider the forward Euler discretization of~\eqref{eq:firstorderODE} with time step $h = \frac{T}{N} >0$, giving
\begin{equation}
{\bf y}_{j+1} = {\bf y}_j + h\, {\bf f}({\bf y}_j, {\boldsymbol{\theta}}_j), \qquad \text{for }j=0,1,\dots,N-1.
\label{eq:firstorderODE_td}
\end{equation}
Equation \eqref{eq:firstorderODE_td} can be seen as the model of a residual neural network with $N$ layers~\cite{He2016}, where ${\bf y}_j, {\bf y}_{j+1} \in \mathbb{R}^{n}$ are the input and output vectors of layer $j$, respectively.
We highlight that, for the discretization of \eqref{eq:firstorderODE}, one could replace the forward Euler method with different discretization methods, hence obtaining different DNN architectures~\cite{Haber_2017}.
In DNNs, ~\eqref{eq:firstorderODE_td} is usually complemented with an output layer ${\bf y}_{N+1} = {\bf f}_N({\bf y}_{N}, {\boldsymbol{\theta}}_{N})$ that depends on the nature of the learning problem (e.g., regression or classification).

DNN training is performed by computing the network weights that minimize a loss function 
\begin{equation*}
\mathcal{L}({\bf y}_{N+1}^1, \ldots, {\bf y}_{N+1}^s, {\boldsymbol{\theta}}),
\end{equation*}
where $\{1, \ldots, s\}$ is the index set of samples used to optimize the network weights and ${\boldsymbol{\theta}}$ collects all DNN parameters.
A remarkable feature of ODE-inspired DNNs is that their properties can be studied, albeit in an approximate way, in terms of the continuous-time nonlinear system~\eqref{eq:firstorderODE}, which is often easier to analyse than~\eqref{eq:firstorderODE_td}~\cite{Haber_2017}.

\subsection{Vanishing/exploding gradients}\label{sec:VanExpGradient}

An obstacle that is commonly faced when training DNNs using gradient based optimization methods, is the problem of exploding/vanishing gradients.
Gradient descent methods update the vector ${\boldsymbol{\theta}}$ as
\begin{equation}
\boldsymbol{\theta}^{(k+1)} = \boldsymbol{\theta}^{(k)} - \gamma \cdot \nabla_{\boldsymbol{\theta}}
\mathcal{L}
\end{equation}
where $\gamma >0$ is the optimization step size.
In particular, by using the chain rule, the gradient of the loss function w.r.t. the parameter $i$ of layer $j$ can be obtained as
\begin{equation}\label{eq:parameterGradient}
\frac{\partial \mathcal{L}}{\partial \theta_{i,j}} = \frac{\partial {\bf y}_{j+1}}{\partial \theta_{i,j}} \frac{\partial \mathcal{L}}{\partial {\bf y}_{j+1}}  =  \frac{\partial {\bf y}_{j+1}}{\partial \theta_{i,j}} 
\left( \prod_{l=j+1}^{N-1} \frac{\partial {\bf y}_{l+1}}{\partial {\bf y}_l} \right)
\frac{\partial \mathcal{L}}{\partial {\bf y}_{N}}.
\end{equation}
The problem of vanishing/exploding gradients is commonly related to the layer gradient magnitudes $\left\| \frac{\partial {\bf y}_{l+1}}{\partial {\bf y}_l}\right\|_2$, $l=0, \ldots, N-1$.
If these terms are all very small, since $\left\| \prod_{l={j+1}}^{N-1} \frac{\partial {\bf y}_{l+1}}{\partial {\bf y}_l}\right\|_2 \le \prod_{l=j+1}^{N-1} \left\|\frac{\partial {\bf y}_{l+1}}{\partial {\bf y}_l} \right\|_2$, the gradient $\frac{\partial \mathcal{L}}{\partial \theta_{i,j}}$ vanishes, and the training stops.
Vice versa, if these terms are very large, $\frac{\partial \mathcal{L}}{\partial \theta_{i,j}}$ becomes very sensitive to perturbations in the vectors $\frac{\partial {\bf y}_{j+1}}{\partial \theta_{i,j}}$ and $\frac{\partial \mathcal{L}}{\partial {\bf y}_N}$, and this can make the learning process unstable or cause overflow issues.
Both problems are generally exacerbated when the number of layers $N$ is large~\cite{GoodBengCour2016}.

For analysing the phenomenon of exploding/vanishing gradients in the context of \textit{recurrent} neural networks, the authors of \cite{Chang19} consider the system~\eqref{eq:firstorderODE} and show that the matrix $\boldsymbol{\phi}(t, 0)=\frac{\partial {\bf y}(t)}{\partial {\bf y}(0)}\in \mathbb{R}^{n\times n}$ obeys the linear time-varying dynamics
\begin{align}
\label{eq:forwardGradientDynamics}
\dot{\boldsymbol{\phi}}(t,0)=\boldsymbol{\mathcal{J}}(t) \boldsymbol{\phi}(t,0), \quad \boldsymbol{\phi}(0,0)={\bf I},
\end{align}
where $\boldsymbol{\mathcal{J}}(t)=\frac{\partial^\top {\bf f}({\bf y},t)}{\partial {\bf y}(t)}$.
In particular, $\boldsymbol{\phi}(t,0)$ can be seen as the continuous-time counterpart of the neural network gradient $\frac{\partial {\bf y}_k}{\partial {\bf y}_0}$, which is similar to the terms appearing in~\eqref{eq:parameterGradient}.
For the sake of simplicity, let us consider the 
simpler case
where $\boldsymbol{\mathcal{J}}(t)=\boldsymbol{\mathcal{J}}$ is time-invariant. 
The properties that $\left\| \boldsymbol{\phi}(t,0)\right\|_2$ neither diverges nor vanishes as $t\rightarrow +\infty$ corresponds to the marginal stability of~\eqref{eq:forwardGradientDynamics}, which is equivalent to requiring that each eigenvalue of $\boldsymbol{\mathcal{J}}$ has zero real part and its geometric and algebraic multiplicity coincide (see, e.g.~\cite{Khalil_NLsys}).
Under suitable assumptions, similar conditions can be also reached if $\boldsymbol{\mathcal{J}}(t)$ varies slowly enough in time~\cite{AscherMetNum}.

\subsection{Anti-symmetric and Hamiltonian-inspired DNNs} \label{sec:existingNets}
Motivated by the goal of mitigating the problem of vanishing/exploding gradients, as well as of having a marginally stable forward dynamics~\eqref{eq:firstorderODE}\footnote{As shown in~\cite{Haber_2017}, this property guarantees reduced sensitivity to perturbations and adversarial attacks on the input features.}, various DNN architectures have been proposed.
They are summarized below, where matrices $\bf K$ and $\bf b$ denote the trainable parameters and $\sigma(\cdot): \mathbb{R}\rightarrow \mathbb{R}$ is an activation function applied element wise to a vector argument.
The network structures are called MS$_i$-DNN ($i=1,2,3$) 
and, for each of them, the underlying ODE as well as the discretization method used are specified.

\begin{itemize}
	\item MS$_1$-DNN \cite{Haber_2017}
	\begin{itemize}
		\small
		\item Layer equation:
		${\bf z}_{j+1} = {\bf z}_j - h \sigma({\bf K}_{j,0}^\top{\bf y}_j + {\bf b}_{j,1})$ and 
		${\bf y}_{j+1} = {\bf y}_j + h \sigma({\bf K}_{j,0}{\bf z}_{j+1} + {\bf b}_{j,2})$ 
		\item Underlying ODE:                
		$\begin{bsmallmatrix}
		\dot{\bf y} \\
		\dot{\bf z}
		\end{bsmallmatrix}
		(t)
		=
		\sigma \left(
		\begin{bsmallmatrix}
		\bf 0 & {\bf K}_0(t)\\
		-{\bf K}_0^\top(t) & \bf 0
		\end{bsmallmatrix}
		\begin{bsmallmatrix}
		{\bf y} \\
		{\bf z} 
		\end{bsmallmatrix}
		(t)
		+\begin{bsmallmatrix}
		{\bf b}_1 \\
		{\bf b}_2 
		\end{bsmallmatrix}
		(t) \right)$
		
		\item Discretization method: Verlet 
	\end{itemize}
	
	\item MS$_2$-DNN \cite{Haber_2017,Chang19}
	\begin{itemize}
		\small
		\item Layer equation:
		${\bf y}_{j+1} = {\bf y}_j + h \sigma({\bf K}_j{\bf y}_{j} + {\bf b}_{j})$ 
		where all ${\bf K}_j$ matrices are skew-symmetric
		
		\item Underlying ODE: 
		$\dot{\bf y}(t) = \sigma({\bf K} (t) {\bf y}(t) + {\bf b} (t))$
		where ${\bf K}(t)$ are skew-symmetric $\forall t\ge 0$
		
		\item Discretization method: forward Euler
	\end{itemize}
	
	\item MS$_3$-DNN \cite{Chang18a}
	\begin{itemize}
		\small
		\item Layer equation:
		${\bf y}_{j+1} = {\bf y}_j + h {\bf K}_{j,1}^\top \sigma({\bf K}_{j,1} {\bf z}_j + {\bf b}_{j,1})$
		and 
		${\bf z}_{j+1} = {\bf z}_j - h {\bf K}_{j,2}^\top \sigma({\bf K}_{j,2} {\bf y}_{j+1} + {\bf b}_{j,2})$
		\item Underlying ODE: 
		$\begin{bsmallmatrix}
		\dot{\bf y} \\
		\dot{\bf z}
		\end{bsmallmatrix}
		(t)
		=
		\begin{bsmallmatrix}
		{\bf K}_1^\top & \bf 0\\
		\bf 0& -{\bf K}_2^\top
		\end{bsmallmatrix} (t)\,
		\sigma \left(
		\begin{bsmallmatrix}
		\bf 0 & {\bf K}_1\\
		{\bf K}_2 & \bf 0
		\end{bsmallmatrix} (t)
		\begin{bsmallmatrix}
		{\bf y} \\
		{\bf z} 
		\end{bsmallmatrix}
		(t)
		+\begin{bsmallmatrix}
		{\bf b}_1 \\
		{\bf b}_2 
		\end{bsmallmatrix}
		(t) \right)$
		
		\item Discretization method: Verlet
	\end{itemize}
\end{itemize}

It is worth to remark that these networks can achieve very good performance on different classification problems including benchmark problems in image classification such as MNIST and CIFAR 10~\cite{Haber_2017,Chang18a,Chang19}.
We highlight that models MS$_1$  and MS$_3$ have been called \textit{Hamiltonian-inspired} in view of their similarity with Hamiltonian models (compare, e.g. the underlying ODE of MS$_1$-DNN and~\eqref{eq:simpleHamiltonianSystems}).
However,~\cite{Haber_2017, Chang18a} do not provide a precise Hamiltonian function for the corresponding ODEs.

\section{Hamiltonian neural networks (H-DNNs)} \label{sec:methods}
In this section, we consider a general class of continuous time Hamiltonian system that we utilize for defining new DNN architectures. We also show that, under weak assumptions, MS$_i$-DNN models, $i=1,2,3$, can be obtained as special cases of the proposed networks. Finally, by assuming constant weights, we analyse marginal stability of the backward gradient dynamics, and provide arguments for supporting the claim that vanishing/exploding gradients are not expected during training.



\subsection{Hamiltonian dynamics}

We consider time-varying Hamiltonian systems \cite{Guo2006} defined by the ODE
\begin{equation}
{\bf \dot y}(t) = {\bf J}({\bf y},t)  \frac{\partial H({\bf y},t)}{\partial {\bf y}} 
\label{eq:TV_HS}
\end{equation}
where the interconnection matrix ${\bf J}({\bf y},t) \in \mathbb{R}^{n\times n}$ is skew-symmetric i.e. ${\bf J}({\bf y},t) = -{\bf J}^\top({\bf y},t)$ $\forall t\geq 0$, and $H({\bf y}(t),t) \in \mathbb{R}$ is the Hamiltonian function. Both $\bf J$ and $H$ are assumed to be smooth functions of all their arguments.

The more common notion of a time-invariant Hamiltonian system \cite{vanderSchaft2017} can be recovered when $\bf J$ and $H$ do not depend upon time.
Time-invariant Hamiltonian systems are marginally stable by construction when $H({\bf y})$ is a positive definite function \cite{Khalil_NLsys}. Therefore, as discussed in Section~\ref{sec:VanExpGradient}, they are a good candidate for defining well-posed DNNs. The same is true for the time-varying model \eqref{eq:TV_HS}, provided that the Hamiltonian changes slowly enough over time.

In the sequel, we focus on the following energy function
\begin{equation}
H({\bf y}(t),t) = \left[ \log(\cosh({\bf K}(t) {\bf y}(t) + {\bf b}(t))) \right]^\top \boldsymbol{1}
\label{eq:nlH}
\end{equation}
where $\log(\cdot)$ and $\cosh(\cdot)$ are applied element-wise, and $\boldsymbol{1} = [1, \dots , 1]^\top$. 
We obtain
\begin{align}
\frac{\partial H({\bf y}(t),t)}{\partial {\bf y}(t)} =& \frac{\partial ({\bf K}(t) {\bf y}(t) + {\bf b}(t))}{\partial {\bf y}(t)}  \frac{\partial H({\bf y}(t),t)}{\partial ({\bf K}(t) {\bf y}(t) + {\bf b}(t))} 
= {\bf K}^\top(t) \tanh({\bf K}(t) {\bf y}(t) + {\bf b}(t)) \label{eq:partialH}
\end{align}
where $\tanh(\cdot)$ is applied element-wise. 
Hence, system~\eqref{eq:TV_HS} becomes
\begin{equation}
\dot{\bf y}(t) = {\bf J}({\bf y},t) {\bf K}^\top(t) \tanh( {\bf K}(t) {\bf y}(t) + {\bf b}(t) ) .
\label{eq:ODE_H}
\end{equation}

\subsection{H-DNNs: relations with existing networks and new architectures}\label{sec:H-DNN_architectures}

We show that the underlying ODEs of the MS$_i$-DNNs (see Section~\ref{sec:existingNets}) are particular instances of \eqref{eq:ODE_H} when $\sigma(\cdot)=\tanh(\cdot)$ and
\begin{itemize}
	\item 	for MS$_1$-DNN,
	${\bf K}(t) = 
	\begin{bsmallmatrix}
	\bf 0 & {\bf K}_0(t)\\
	-{\bf K}_0^\top(t) & \bf 0
	\end{bsmallmatrix}
	$
	is invertible $\forall t \geq 0$
	and
	${\bf J}({\bf y},t)  {\bf K}^\top(t) = {\bf I}$,
	\item	for MS$_2$-DNN,
	${\bf K}(t) = 
	- {\bf K}^\top(t)$
	is invertible
	$\forall t \geq 0 $
	and
	${\bf J}({\bf y},t)  {\bf K}^\top(t) = {\bf I}$,
	
	\item	for MS$_3$-DNN,
	${\bf K}(t) = 
	\begin{bsmallmatrix}
	\bf 0 & {\bf K}_1(t)\\
	{\bf K}_2(t) & \bf 0
	\end{bsmallmatrix}$
	and
	${\bf J}({\bf y},t)  =
	\begin{bsmallmatrix}
	\bf 0 & {\bf I}\\
	-{\bf I} & \bf 0
	\end{bsmallmatrix}$.
\end{itemize}

A necessary condition for the skew-symmetric $n \times n$ matrix ${\bf K}(t)$ to be invertible is that the size $n$ of input features is even\footnote{For a $n\times n$ skew-symmetric matrix ${\bf A}$ we have, $\det({\bf A}) = \det({\bf A}^\top) = \det({\bf A}^{-1}) = (-1)^n \det({\bf A})$. If $n$ is odd, then $\det({\bf A}) = - \det({\bf A}) = 0$. Thus, ${\bf A}$ is not invertible.}.
If $n$ is odd, however, one can perform input-feature augmentation by adding an extra state initialized at zero to satisfy the previous condition \cite{dupont2019augmented}.

Next, we introduce two new architectures (called H$_i$-DNN, $i=1,2$) stemming from \eqref{eq:ODE_H} when ${\bf J}({\bf y},t)$ is constant and forward Euler discretization with step $h>0$ is applied.
The resulting layer equations are
\begin{equation}
{\bf y}_{j+1} = {\bf y}_j + h\, {\bf J} {\bf K}^\top_j \tanh({\bf K}_j {\bf y}_j +{\bf b}_j) \qquad j=0,1,\cdots, N-1
\label{eq:H-DNN}
\end{equation}
where we set 
${\bf J}({\bf y},t)  =
\begin{bsmallmatrix}
\bf 0 & {\bf I}\\
-{\bf I} & \bf 0
\end{bsmallmatrix}$ for H$_1$-DNN,
and ${\bf J}({\bf y},t) =
\begin{bsmallmatrix}
0  & 1 & \dots & 1\\
-1 & 0 & \dots & 1\\
\vdots & \vdots & \ddots & \vdots \\
-1 & -1 & \dots & 0\\
\end{bsmallmatrix}$ for H$_2$-DNN.

In spite of the specific choices of $\bf J$, both DNNs contain more trainable parameters in each layer than MS$_i$-DNN, $i=1,2,3$. In this sense, they are more expressive than MS$_i$ architectures and, as shown later in Section \ref{sec:simple_examples}, one can use less layers while obtaining similar prediction accuracy.

\begin{remark}
	Although it is not guaranteed that Euler discretization preserves marginal stability of the continuous-time dynamics, it leads to simpler layer equations compared to more sophisticated discretization approaches, and can achieve good performance on benchmark examples (see Section \ref{sec:simple_examples}).
	Moreover the discretization accuracy can be controlled through $T$ and $h$. 
	These features may be attractive to practitioners. 
\end{remark}


\subsection{Training algorithm}\label{sec:regularization}

For all DNN architectures introduced in Section~\ref{sec:H-DNN_architectures}, we consider multicategory classification problems where $M$ is the number of classes, and the input features and their corresponding true labels are $({\bf y}_0^k,c^k), k=1,\dots,s, c^k\in\{0,\dots,M-1\}.$ The networks are trained by solving the optimization problem
\begin{align}
\min_{\boldsymbol{\theta}} & \qquad \frac{1}{s} \sum_{k=1}^s \mathcal{L}({\bf f}_N({\bf y}^k_{N}), c^k) + \alpha_c \, R_N(\boldsymbol{\theta}_N) + \alpha \, R({\bf K}_{0,\dots,N-1}, {\bf b}_{0,\dots,N-1}) \nonumber\\
\text{s.t.}& \qquad {\bf y}^k_{j+1} = {\bf y}^k_j + h\, {\bf J}_j({\bf y}^k_j) {\bf K}^\top_j \tanh({\bf K}_j {\bf y}^k_j +{\bf b}_j), \quad j=0,1,\dots,N-1
\label{eq:minimization}
\end{align}
where $R_N(\cdot)$ is the $L_2$ regularization term of the output layer\footnote{For a two class classification problem, it is given by $R_N(\cdot) = \left\| {\bf W} \right\|^2 + \mu^2$. For multicategory problems, we refer the reader to \cite{Haber_2017}.}
and $R({\cdot})$ is the regularization term of layers $0,\dots, N-1$.
The output layer is problem dependent, e.g. for a two class classification problem, it is usually given by $f_N({\bf y}_N, \boldsymbol{\theta}_N) = \sigma_c ({\bf Wy}_N + {\mu})$ where ${\bf W}\in \mathbb{R}^{1\times n}$, $\mu \in \mathbb{R}$, $\boldsymbol{\theta}_N=({\bf W}, \mu)$ and $\sigma_c(x) = \frac{1}{1+e^{-x}}$ is the logistic function.
The minimization is done over $\boldsymbol{\theta}$, i.e., all the trainable parameters that define the network $\{{\bf K}_{0,\dots,N-1}, {\bf b}_{0,\dots,N-1}, {\bf W}, \mu\}$.

Following the work in \cite{Haber_2017} and \cite{Chang18a}, we define the regularization term for the H-DNNs as $R = R_K({\bf K}_{0,\dots,N-1}) + R_b({\bf b}_{0,\dots,N-1})$, where
\begin{equation}
R_K({\bf K}_{0,\dots,N-1}) = \frac{h}{2}\sum_{j=1}^{N-1} \left\| {\bf K}_j-{\bf K}_{j-1} \right\|^2_F \quad \text{and}\quad R_b({\bf b}_{0,\dots,N-1}) = \frac{h}{2}\sum_{j=1}^{N-1} \left\| {\bf b}_j-{\bf b}_{j-1} \right\|^2
\label{eq:ruthottoReg}
\end{equation}
so as to favour weights that vary smoothly between adjacent layers. The coefficients $\alpha \geq 0$ and $\alpha_c \geq 0$\footnote{$\alpha_c$ is usually called \textit{weight decay}.} are hyperparameters that represent the trade-off between fitting and regularization.

\subsection{Stability of the backward gradient dynamics} \label{sec:eigenvalues}

As discussed in Section~\ref{sec:VanExpGradient}, to avoid vanishing/exploding gradients, we would like to ensure that the following terms are not vanishing nor exploding
\begin{equation}
\left( \prod_{l=j+1}^{N-1} \frac{\partial {\bf y}_{l+1}}{\partial {\bf y}_l} \right) = \frac{\partial {\bf y}_N}{\partial {\bf y}_{j+1}} \text{ for } j=N-2, \dots, 0.
\label{eq:norm_grad}
\end{equation}
This analysis can also be tackled from the continuous-time perspective by considering \eqref{eq:firstorderODE} and noting that \eqref{eq:norm_grad} corresponds to $\frac{\partial {\bf y}(T)}{\partial {\bf y}(T-t)}$, where $t = h(j+1)$, $T=h(N-1)$ and $h$ is the step size.

We call the evolution of $\boldsymbol{\phi}(T,T-t) \triangleq \frac{\partial {\bf y}(T)}{\partial {\bf y}(T-t)}$ the \textit{backward gradient dynamics} because $T-t$ decreases from $T$ to zero as $t$ increases. This term is different from the one considered in \eqref{eq:forwardGradientDynamics} which captures the sensitivity to input features and not the evolution across layers of gradients appearing in backpropagation.

Our next goal is to obtain the dynamics of $\boldsymbol{\phi}(T,T-t)$ for the Hamiltonian model \eqref{eq:ODE_H}. We start from the simple case where the parameters of \eqref{eq:ODE_H} are constant, i.e. ${\bf J}({\bf y}(t),t) = {\bf J}$, ${\bf K}(t) = {\bf K}$, ${\bf b}(t) = {\bf b}$, and call the corresponding networks \textit{time-invariant H-DNNs}. The following Lemma, whose proof can be found in Appendix~\ref{ap:lemma_grad}, provides the desired model.



\begin{lemma}\label{lem:grad_dynamics}
	Given the time-invariant ODE $\dot{\bf y}(t) = {\bf f}({\bf y}(t))$, the time evolution of $\boldsymbol{\phi}(T,T-t)$ is given by 
	\begin{equation}
	\frac{d}{d t} \boldsymbol{\phi}(T,T-t) = \left. \frac{\partial {\bf f}}{\partial {\bf y}} \right\rvert_{{\bf y}(T-t)} \boldsymbol{\phi}(T,T-t)\,, \quad \boldsymbol{\phi}(T,T) = {\bf I}.
	\label{eq:BackwardGradientDynamics}
	\end{equation}
\end{lemma} 
Since in our case ${\bf f}({\bf y}(t)) = {\bf J} {\bf K}^\top \tanh( {\bf K} {\bf y}(t) + {\bf b} ) $, we have
\begin{align}
\frac{\partial {\bf f}}{\partial {\bf y}} &=
\frac{\partial}{\partial {\bf y}} \left(  {\bf K}^\top \tanh({\bf K} {\bf y} +{\bf b})\right) {\bf J}^\top= 
\frac{\partial}{\partial {\bf y}} \left( \tanh({\bf K} {\bf y} +{\bf b})\right) {\bf K}{\bf J}^\top \nonumber \\
&= {\bf K}^\top \text{diag}\left(\tanh'({\bf K} {\bf y} +{\bf b})\right) {\bf K}{\bf J}^\top = {\bf K}^\top {\bf D}({\bf y}) {\bf K}{\bf J}^\top
\label{eq:Jacobian}
\end{align}
where ${\bf D}({\bf y}) = \text{diag}\left(\tanh'({\bf K} {\bf y} +{\bf b})\right)$ and $\tanh'(\cdot)$ computes element-wise the derivative of $\tanh(\cdot)$. 

The next two lemmas, proved in Appendix~\ref{ap:lemma_eig} and \ref{ap:lemma_multiplicity}, show that the Jacobian matrix~\eqref{eq:Jacobian} satisfies the conditions for marginal stability.

\begin{lemma}\label{lem:eig_imag}
	The eigenvalues of $\frac{\partial {\bf f}}{\partial {\bf y}}$ are purely imaginary.
\end{lemma}

\begin{lemma}\label{lem:multiplicity}
	The algebraic and geometric multiplicity of each eigenvalue of $\frac{\partial {\bf f}}{\partial {\bf y}}$ do coincide.
\end{lemma}

As shown in the corresponding proofs, Lemma~\ref{lem:eig_imag} hinges on results available in \cite{Chang19}. 
However, the multiplicity of the eigenvalues of $\frac{\partial \bf f}{\partial {\bf y}}$ (Lemma~\ref{lem:multiplicity}) has not been analysed in previous publications.


If ${\bf D}({\bf y}(t))$ is constant, Lemmas~\ref{lem:eig_imag} and \ref{lem:multiplicity} imply that $\boldsymbol{\phi}(T,T-t)$, neither diverges nor converges to zero, irrespectively of the weights $\bf J$, $\bf K$ and $\bf b$ and the final time $T$. This property, however, can be compromised by the time-varying nature of ${\bf D}({\bf y}(t))$ and the weights (${\bf K}(t)$ and ${\bf b}(t)$) as well as the time discretization process underlying \eqref{eq:norm_grad}. Nevertheless, Lemmas~\ref{lem:eig_imag} and \ref{lem:multiplicity} suggest that if ${\bf D}({\bf y}(t))$, ${\bf K}(t)$ and ${\bf b}(t)$ change slowly enough and $h$ is sufficiently small, the growth or decrease of the terms $\left\|\frac{\partial {\bf y}_{N}}{\partial {\bf y}_{j+1}} \right\|_2$ can be kept under control.


While we do not provide a complete theoretical analysis when ${\bf D}({\bf y}(t))$, ${\bf K}(t)$ and ${\bf b}(t)$ are time varying, we simply highlight that changes in the parameters ${\bf K}_j$ and ${\bf b}_j$ across the layers of H$_i$-DNN, $i=1,2$, can be controlled by suitably choosing the regularization parameter $\alpha$ in \eqref{eq:minimization}.
Moreover, in Section \ref{sec:gradient}, we provide a simulation study of the backward gradient dynamics confirming the absence of vanishing/exploding gradients when ${\bf K}_j$ and ${\bf b}_j$ are not constant.

\section{Numerical examples}\label{sec:examples}

\subsection{Binary classification examples}\label{sec:simple_examples}

We test MS- and H-DNNs introduced in Section~\ref{sec:methods} on two benchmark examples (the ``Swiss roll'' and the ``double moons'' datasets in Figures \ref{fig:examplesClassification_swissroll} and \ref{fig:examplesClassification_doublemoons}) concerning binary classification with features in $\mathbb{R}^2$.

As in \cite{Haber_2017}, we consider MS- and H-DNNs with augmented input features \cite{dupont2019augmented} so as to increase the modelling power.
More specifically, input feature vectors are given by $\begin{bsmallmatrix}({\bf y}_0^k)^\top &0 & 0 \end{bsmallmatrix}^\top \in \mathbb{R}^4$ where ${\bf y}_0^k \in \mathbb{R}^2, k=1,\dots,s$ are the input datapoints (see Figures \ref{fig:examplesClassification_swissroll} and \ref{fig:examplesClassification_doublemoons}).
We complement the DNNs with an output layer ${y}_{N+1} = f_N({\bf y}_N, \boldsymbol{\theta}_N)$ (see Section \ref{sec:regularization}). 
The optimization problem is solved using the Adam algorithm with minibatches (see Appendix~\ref{ap:implementation_2d} for details), and standard cross-entropy \cite{GoodBengCour2016} as the loss function $\mathcal{L}$ in \eqref{eq:minimization}.

\begin{table}
	\begin{center}
		\caption{Classification accuracies over test sets for different examples using different network structures with 4 neurons (nf) each layer. The first three columns represent the existing structures while in the two last columns we present the results for the new H-DNNs. The first two best accuracies in each row are in bold. Last row presents the number of parameters per layer of each network.}
		\label{tab:classif}
		\footnotesize
		\begin{tabular}{r|c|c|c|c||c|c}
			& & MS$_1$-DNN & MS$_2$-DNN & MS$_3$-DNN & H$_1$-DNN & H$_2$-DNN \\
			\hline
			Swiss       & 4 layers & 77.1\% & 79.7\% & 90.1\% & \textcolor{dark-gray}{\textbf{93.6\%}} & \textbf{98.9\% }\\ 
			roll           & 8 layers & 91.5 \% & 90.7\% & 87.0\% & \textcolor{dark-gray}{\textbf{99.0\%}} & \textbf{99.4\%} \\ 
			& 16 layers & 97.7\% & 99.7\% & 97.1\% & \textbf{99.8\%} & \textbf{99.8\%} \\ 
			& 32 layers & \textbf{100\%} & \textbf{100\%} & 98.4\% & \textcolor{dark-gray}{\textbf{99.8\%}} & 99.2\%\\  
			& 64 layers & \textbf{100\%} & \textbf{100\%} & \textbf{100\%} & \textcolor{dark-gray}{\textbf{99.8\%}} & \textbf{100\%} \\ 
			\hline
			Double & 1 layer       & 92.5\% & 91.3\% & 97.6\% & \textbf{100\%} & \textcolor{dark-gray}{\textbf{99.9\%}} \\ 
			moons & 2 layers      & 98.2\% & 94.9\% & \textcolor{dark-gray}{\textbf{99.8\%}} & \textbf{100\%} & \textbf{100\%} \\ 
			& 4 layers      & \textcolor{dark-gray}{\textbf{99.5\%}} & \textbf{100\%} & \textbf{100\%} & \textbf{100\%} & \textbf{100\%} \\ 
			\hline
			\hline
			\multicolumn{2}{c|}{\textit{\# parameters per layer}} & $\frac{\text{nf}^2}{4} + \text{nf}$ & $\frac{\text{nf}^2 + \text{nf}}{2}$ & $\frac{\text{nf}^2}{2} + \text{nf}$ & $\text{nf}^2 + \text{nf}$ & $\text{nf}^2 + \text{nf}$\\
			\hline
		\end{tabular}
	\end{center}
\end{table}

In Table~\ref{tab:classif}, we present the classification accuracies over test sets for the network structures in Section \ref{sec:methods} with different number of layers. 
It can be seen, for a fixed number of layers, that the performances of H$_1$-DNN and H$_2$-DNN are similar or better compared to the other networks. 
This can be motivated by the fact that, as discussed in Section~\ref{sec:H-DNN_architectures}, the new architectures are more expressive than MS$_i$-DNNs.
We indicate in the last row of Table~\ref{tab:classif}, the number of parameters per layer of each network. Note that networks with same number of parameters have similar performance.

The coloured regions in Figure~\ref{fig:examplesClassification} show the predictive power of an example network (H$_1$-DNN). It can be noticed that the datapoints do not lie close to the decision boundary, hence  confirming the robustness of classification against perturbation of input features.

\begin{figure}
	\centering
	\begin{subfigure}{.4\textwidth}
		\centering
		\includegraphics[width=\linewidth]{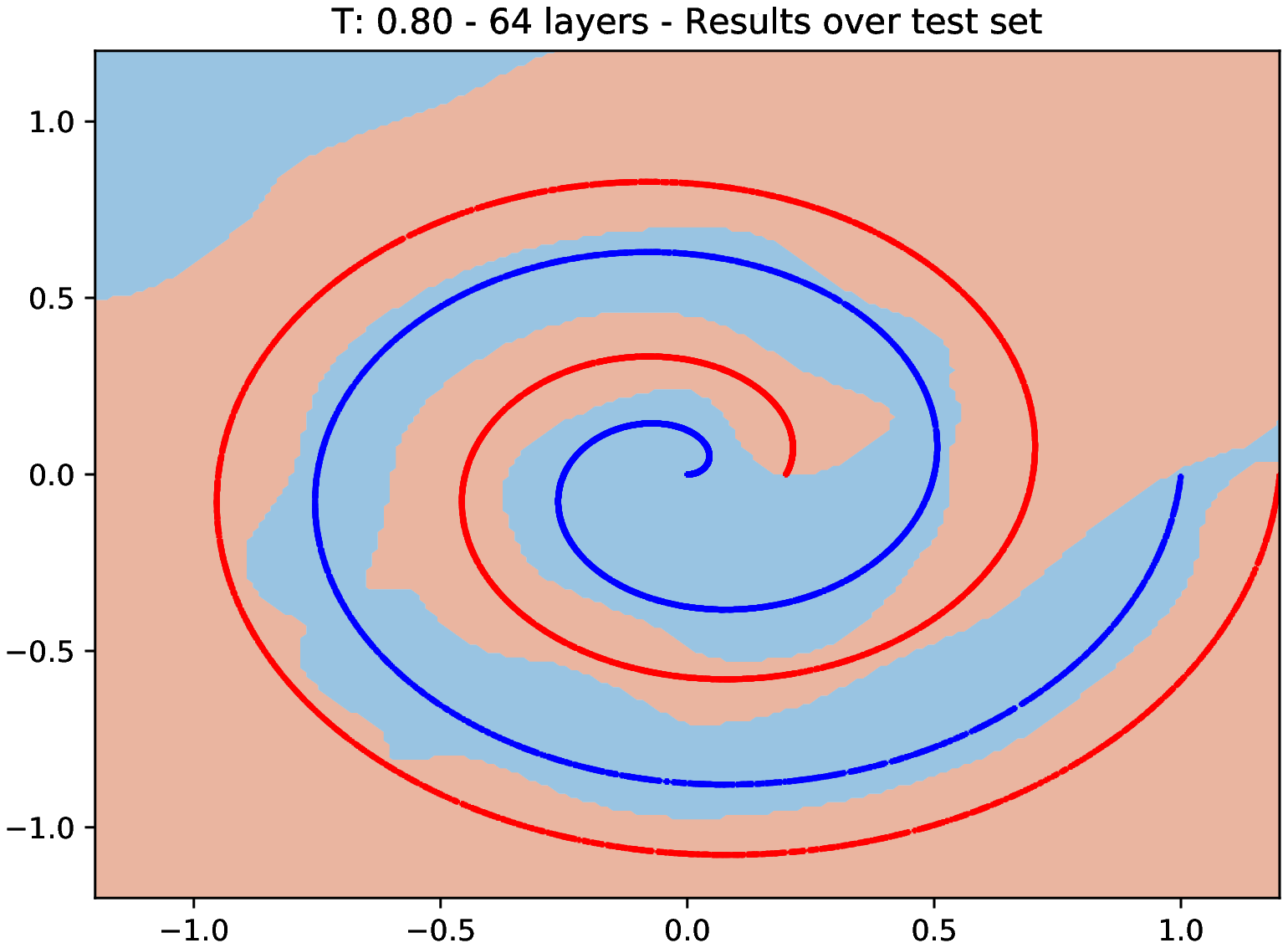}  
		\caption{}
		\label{fig:examplesClassification_swissroll}
	\end{subfigure}%
	\begin{subfigure}{.4\textwidth}
		\centering
		\includegraphics[width=\linewidth]{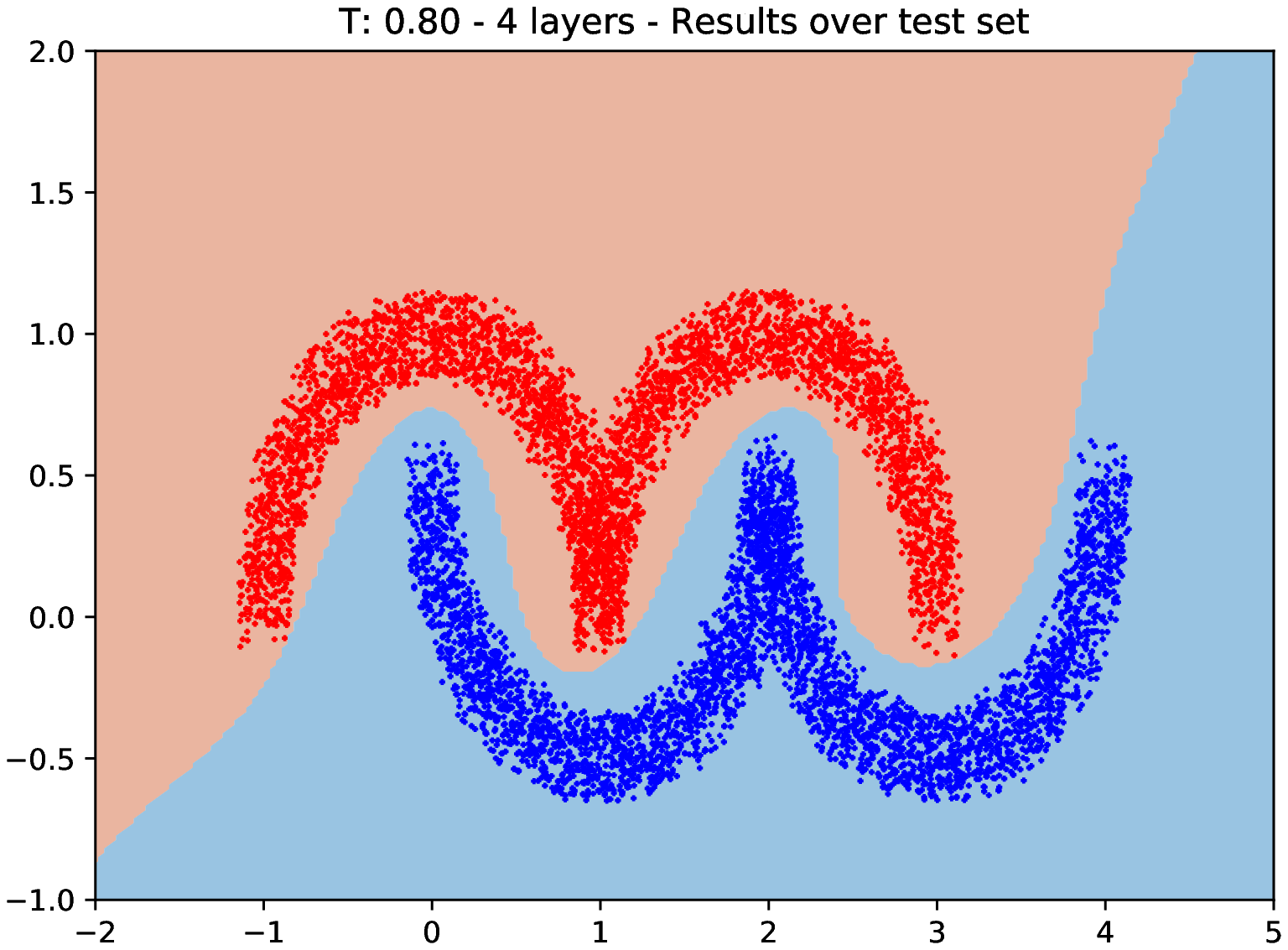}  
		\caption{}
		\label{fig:examplesClassification_doublemoons}
	\end{subfigure}%
	\caption{Results for the H$_1$-DNN architecture with (\textit{a}) 64 and (\textit{b}) 4 layers. Labelled datapoints for (\textit{a}) ``Swiss roll'' and (\textit{b}) ``double moons''. Coloured regions representing the predictions of the trained DNNs.}
	\label{fig:examplesClassification}
\end{figure}

\subsection{Experiments with the MNIST dataset}\label{sec:mnist}

We evaluate our methods on a standard image classification benchmark: MNIST\footnote{\url{http://yann.lecun.com/exdb/mnist/}}.

The dataset consists of $28 \times 28$ digital images in gray scale of hand-written digits from 0 to 9 with their corresponding labels. It contains 60,000 train examples and 10,000 test examples.

Following \cite{Haber_2017}, the network architecture consists of a convolutional layer followed by a Hamiltonian DNN and an output layer. 
The convolutional layer is a linear transformation that expands the data from 1 to 8 channels, and the output layer uses all the output values (i.e., no pooling is performed) for a linear transformation plus a softmax activation function to obtain a vector in $\mathbb{R}^{10}$ that represents the probabilities of the data to belong to each of the 10 classes.

For the Hamiltonian DNN, we use MS$_1$-DNNs and H$_2$-DNNs\footnote{Similar results can be obtained using other MS or H-DNNs.} with 2, 4, 8 and 16 layers. 
We set $h = 0.4$ for MS$_1$-DNNs and $h = 0.05$ for H$_2$-DNNs.
Moreover, we include as a baseline, the results obtained when omitting the Hamiltonian DNN block, i.e., when considering only a convolutional layer followed by the output layer.
The implementation details can be found in Appendix~\ref{ap:implementation_mnist}.


Table~\ref{tab:mnist}, summarizing the train and test accuracies of these networks, shows that both network structures achieve similar performance.
Note that, while the training errors are almost zero, the test errors are reduced when incrementing the number of layers.
Moreover, these results are in line with test accuracies obtained when using standard convolutional layers instead of Hamiltonian DNNs.

\begin{table}
	\footnotesize
	\begin{center}
		\caption{Classification accuracies over train and test sets for MNIST example using MS$_1$-DNN and H$_2$-DNN.}
		\label{tab:mnist}
		\begin{tabular}{|c|c c|c c|}
			\hline
			Number of& \multicolumn{2}{c}{MS$_1$-DNN} & \multicolumn{2}{|c|}{H$_2$-DNN} \\ \cline{2-5} 
			layers & Train & Test & Train & Test \\ \hline
			\hline
			0 &  93.730\% & 92.47\% & 93.828\% & 92.41\% \\ 
			2 & 99.570\% & 97.72\% & 99.815\% & 97.83\% \\   
			4 & 99.970\% & 98.03\% & 99.789\% & 98.02\% \\  
			8 &  99.982\% & 98.05\% & 99.707\% & 98.22\% \\ 
			16 & 100\% & 98.14\% & 99.503\% & 98.21\% \\ 
			\hline
		\end{tabular}
	\end{center}
\end{table}

\subsection{Gradient analysis} \label{sec:gradient}

In order to provide evidence of good numerical results during training, we analyse the evolution of the terms \eqref{eq:norm_grad} for deep networks.
At each optimization step, the gradient of the loss function with respect to each of the parameters is calculated (backward propagation).
In this analysis, we study the Jacobian matrices $\frac{\partial {\bf y}_N}{\partial {\bf y}_{j+1}}$ in \eqref{eq:norm_grad}, and we plot in Figure \ref{fig:gradients} their norms for some layers and for each of the 960 iterations composing the training process when using H-DNNs.

Using ``Double moons'' example, we train a 64-layer H$_1$-DNN, a variant of the same network where we impose the parameters of all layers to coincide, i.e. a \textit{time-invariant} network with ${\bf K}(t)={\bf K}$ and ${\bf b}(t) = {\bf b}$ and a fully-connected neural network (FCNN) with 32 layers\footnote{See Appendix~\ref{ap:implementation_fc} for implementation details.}.

For the H-DNNs case, it can be seen that  the norms of the terms $\frac{\partial {\bf y}_N}{\partial {\bf y}_{j+1}}$ for $j=0,10,\cdots,60$ are bounded in the intervals $[1,17]$ and $[1,37]$ during the whole training. Results are similar when using deeper networks.
Although it is not shown, we highlight that the gradients $\frac{\partial \mathcal{L}}{\partial \theta_{ij}}$ do converge to zero in approximately 500 and 300 iterations respectively, showing that the optimization algorithm has achieved a (possible local) minimum.
When using FCNN, however, it can be shown that gradient norms quickly tends to zero once the network is deep enough. 
For instance, for the 32-layer FCNN, the training stops in approximately 200 iterations and the final test accuracy is only 50\%.

\begin{figure}
	\centering
	\begin{subfigure}{.4\textwidth}
		\centering
		\includegraphics[width=1\linewidth]{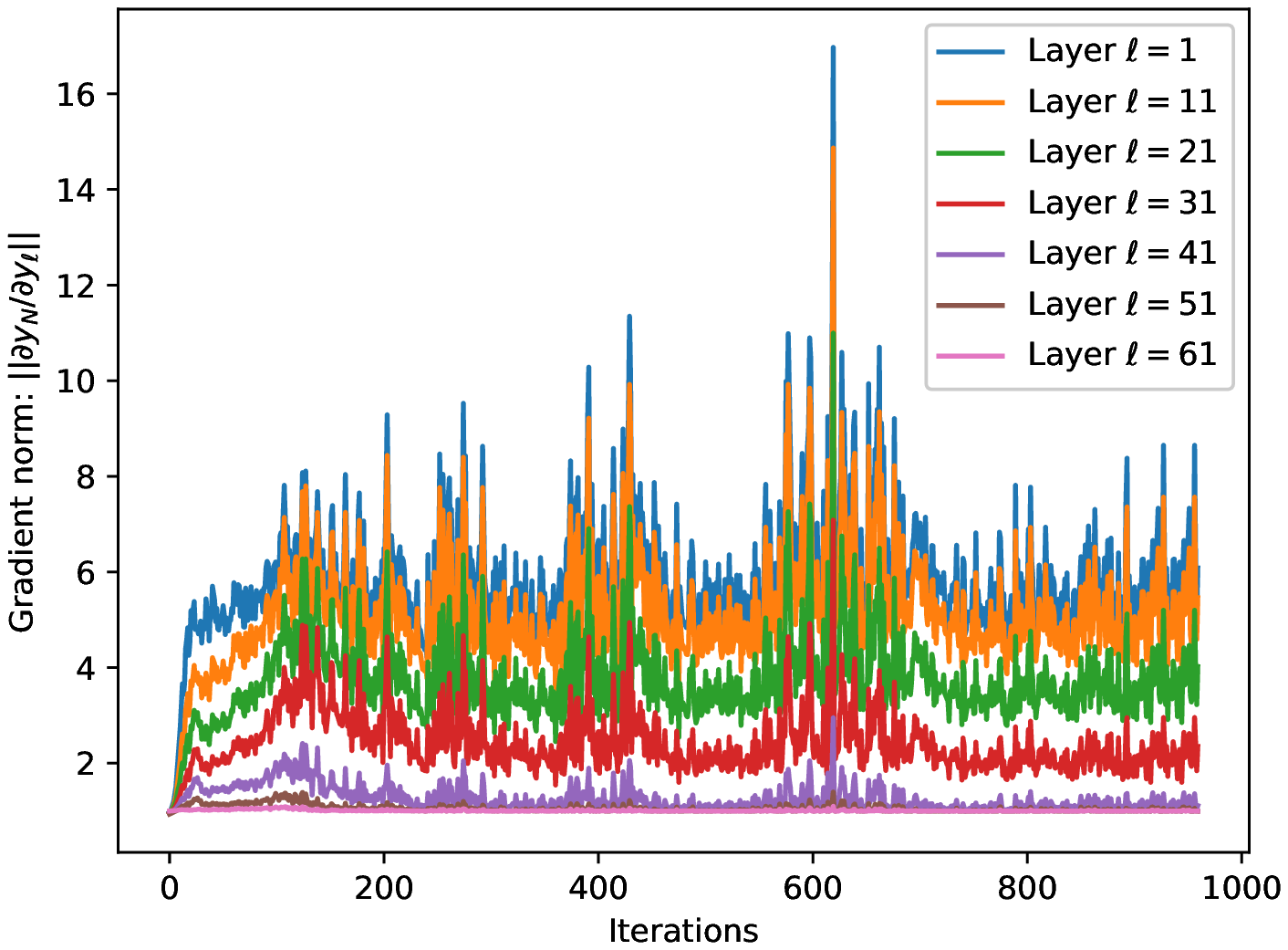} 
		\caption{}
	\end{subfigure}%
	\begin{subfigure}{.4\textwidth}
		\centering
		\includegraphics[width=1\linewidth]{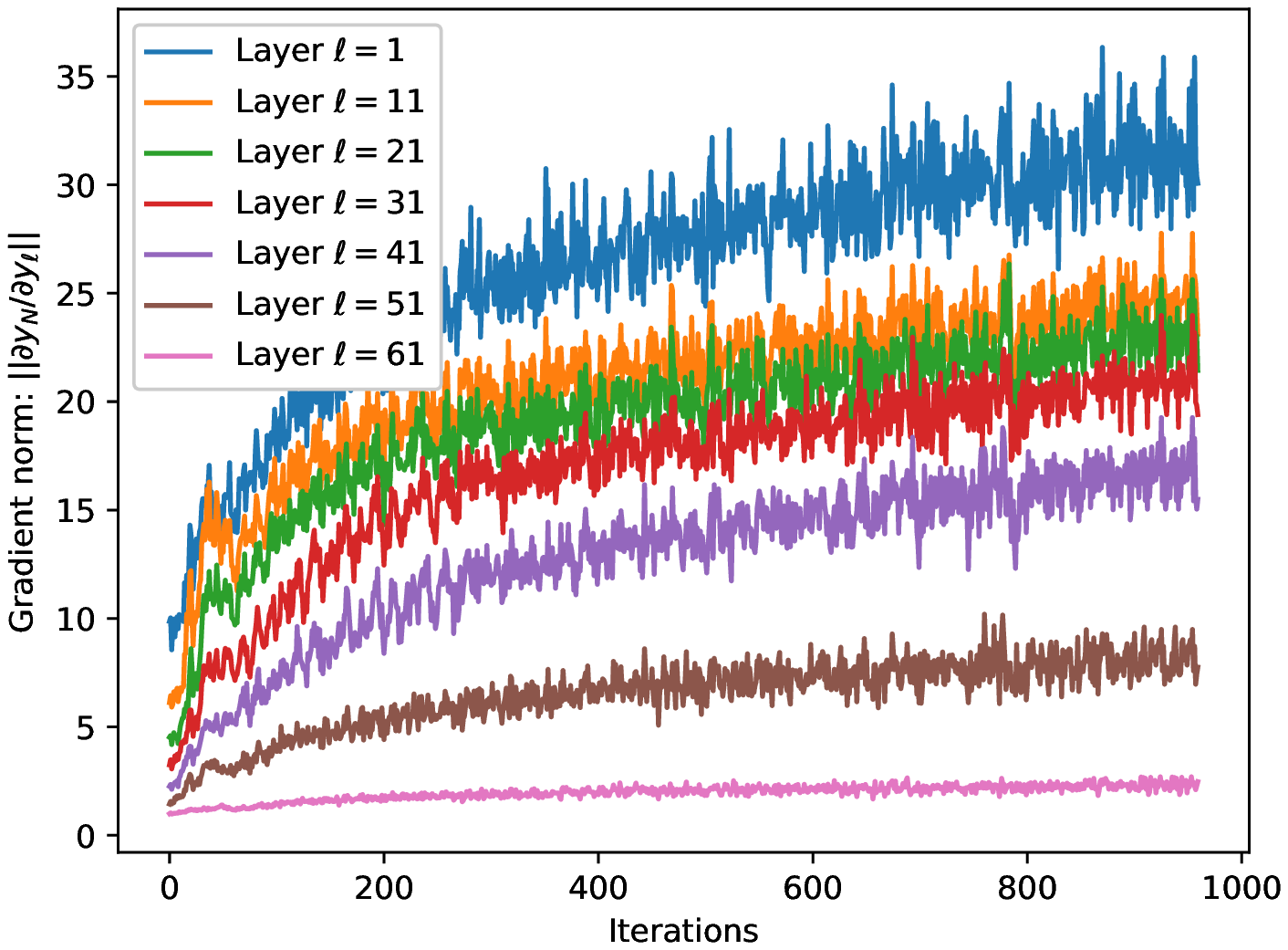} 
		\caption{}
	\end{subfigure}%
	\caption{Evolution of the 2-norm of $\frac{\partial {\bf y}_N}{\partial {\bf y}_\ell}$, $\ell=1,11,21,31,41,51,61$, during the training (960 iterations) of a 64-layer (\textit{a})  H$_1$-DNN and (\textit{b}) \textit{time-invariant} H$_1$-DNN.}
	\label{fig:gradients}
\end{figure}

\section{Conclusions} \label{sec:conclusion}

We present a unified framework for DNNs based on Hamiltonian systems which encompasses existing classes of marginally stable networks.
We define two new DNN structures, which are more flexible than existing ones, while having similar of better performance.
We present the analysis of the backward gradient dynamics for \textit{time-invariant} DNNs and show a simulation study for the \textit{time-varying} case.

Our work is a first step towards the design of new families of H-DNNs since different Hamiltonian energy functions originate new architectures. Future research will also focus on the use of different discretization schemes for defining alternative layer equations.


\section*{Acknowledgements}
Research supported by the Swiss National Science Foundation under the NCCR Automation (grant agreement 51NF40\_180545).

\newpage

\bibliographystyle{IEEEtran}
\bibliography{references}

\newpage

\appendix
\section{Proofs}

\subsection{Proof of Lemma \ref{lem:grad_dynamics}}\label{ap:lemma_grad}
Given the ODE
\begin{equation}
\dot{\bf y} = {\bf f}({\bf y}), \quad {\bf y}(0) = {\bf y}_0,
\label{eq:ap_ode}
\end{equation}
we want to calculate the dynamics of$\frac{\partial {\bf y}(T)}{\partial {\bf y}(T-t)}$.

\begin{proof}
	The solution to \eqref{eq:ap_ode} is
	\begin{equation}
	{\bf y}(t)={\bf y}(0) + \int_0^t {\bf f}({\bf y}(\tau)) d \tau
	\end{equation}
	Evaluating \eqref{eq:ap_ode} in $t=T$ and $t=T-t$ and subtracting them, we obtain
	\begin{equation*}
	{\bf y}(T) = {\bf y}(T-t) + \int_{T-t}^T {\bf f}({\bf y}(\tau)) d \tau
	= {\bf y}(T-t) + \int_{0}^t {\bf f}({\bf y}(s+T-t)) d s
	\end{equation*} 
	where to obtain the second equality, we change the integration variable by defining $\tau=s+T-t$.
	Therefore, we have
	\begin{equation*}
	\frac{\partial {\bf y}(T)}{\partial {\bf y}(T-t)} = {\bf I} + \frac{\partial \int_0^t {\bf f}({\bf y}(s+T-t)) ds }{\partial {\bf y}(T-t)}
	= {\bf I} + \int_0^t \frac{\partial {{\bf y}(s+T-t)}}{\partial {\bf y}(T-t)}\, \left.\frac{\partial {\bf f}}{\partial {\bf y}}\right\rvert_{{\bf y}(s+T-t)} ds
	\end{equation*}
	As a result we have that
	\begin{align*}
	\frac{\partial {\bf y}(T)}{\partial {\bf y}(T-t-\delta)}&=   \frac{\partial {\bf y}(T-t)}{\partial {\bf y}(T-t-\delta)}   \frac{\partial {\bf y}(T)}{\partial {\bf y}(T-t)} \\
	&=  \left( {\bf I}+ \int_0^\delta \frac{\partial {{\bf y}(s+T-t-\delta)}}{\partial {\bf y}(T-t-\delta)}\, \frac{\partial {\bf f}}{\partial {\bf y}}\Big\rvert_{{\bf y}(s+T-t-\delta)} ds \right)   \frac{\partial {\bf y}(T)}{\partial {\bf y}(T-t)}
	\end{align*}
	Therefore, it follows that
	\begin{align*}
	\frac{\partial {\bf y}(T)}{\partial {\bf y}(T-t-\delta)}-  \frac{\partial {\bf y}(T)}{\partial {\bf y}(T-t)}
	=    \left( \int_0^\delta \frac{\partial {{\bf y}(s+T-t)}}{\partial {\bf y}(T-t-\delta)} \, \frac{\partial {\bf f}}{\partial {\bf y}}\Big\rvert_{{\bf y}(s+T-t-\delta)} ds \right)     \frac{\partial {\bf y}(T)}{\partial {\bf y}(T-t)}
	\end{align*}
	Dividing both sides by $\delta$ and taking the limit for $\delta \rightarrow 0$, we obtain
	\begin{align} 
	\frac{d}{d t} \frac{\partial {\bf y}(T)}{\partial {\bf y}(T-t)} = \frac{\partial {\bf f}}{\partial {\bf y}}\Big\rvert_{{\bf y}(T-t)} \frac{\partial {\bf y}(T)}{\partial {\bf y}(T-t)}.
	\end{align}
\end{proof}

\subsection{Proof of Lemma \ref{lem:eig_imag}}\label{ap:lemma_eig}
\begin{proof}
	Since ${\bf J}^\top$ is a skew-symmetric matrix, ${\bf K}{\bf J}^\top {\bf K}^\top$ is also skew-symmetric. By following the proof in the Appendix of \cite{Chang18a}, we have that the eigenvalues of ${\bf K}{\bf J}^\top {\bf K}^\top {\bf D}$ are all imaginary.
	Since $\text{eig} \left( {\bf K}^\top {\bf D} {\bf K} {\bf J}^\top \right)= \text{eig} \left(  {\bf K}{\bf J}^\top {\bf K}^\top {\bf D} \right)$, we obtain that the eigenvalues of  $\left. \frac{\partial \bf f}{\partial {\bf y}} \right\rvert_{{\bf y}(T-t)}$ are all imaginary.
\end{proof}

\subsection{Proof of Lemma \ref{lem:multiplicity}}\label{ap:lemma_multiplicity}
We introduce the following Lemma which is needed to prove Lemma~\ref{lem:multiplicity}. In the following, ${\bf Q}^\top$ and ${\bf Q}^*$ denote the transpose and the conjugate transpose of ${\bf Q}$, respectively.

\begin{lemma}\label{lem:diag}
	Given a real skew-symmetric matrix ${\bf Q} = -{\bf Q}^\top \in \mathbb{R}^{n\times n}$ and a positive definite matrix ${\bf P}\in \mathbb{R}^{n\times n}$, we have ${\bf PQ}$ is diagonalizable.
\end{lemma}

\begin{proof}
	Assume, by contradiction, that ${\bf PQ}$ is not diagonalizable.
	Then it has an eigenvalue $\lambda$ with a corresponding eigenvector ${\bf x}\neq {\bf 0}$ and a corresponding generalized eigenvector ${\bf z}\neq {\bf 0}$ such that
	\begin{align*}
	({\bf PQ} -\lambda {\bf I}){\bf x}={\bf 0}, \quad
	({\bf PQ} - \lambda {\bf I}){\bf z}={\bf x}
	\end{align*}
	Therefore, since  ${\bf P}$ is invertible and we have
	\begin{align*}
	({\bf Q}-\lambda {\bf P}^{-1}){\bf x}={\bf 0},  \quad
	({\bf Q}-\lambda {\bf P}^{-1} ){\bf z}={\bf P}^{-1}{\bf x}
	\end{align*}
	By left multiplying the above equations with ${\bf z}^*$ and ${\bf x}^*$ respectively, we have
	\begin{align}
	{\bf z}^*({\bf Q} - \lambda {\bf P}^{-1}) {\bf x} & = 0 \label{eq.1}\\
	{\bf x}^*({\bf Q} - \lambda {\bf P}^{-1} ){\bf z} & = {\bf x}^*{\bf P}^{-1}{\bf x} \label{eq.2}
	\end{align}
	Taking the conjugate transpose of LHS of (\ref{eq.1}), we have
	\begin{align*}
	{\bf x}^*({\bf Q}^*-\lambda^* {\bf P}^{-1}){\bf z}=-{\bf x}^*{\bf Q}{\bf z}-\lambda^* {\bf x}^*{\bf P}^{-1}{\bf z} = 0
	\end{align*}
	Therefore, we have	${\bf x}^*{\bf Q}{\bf z}=-\lambda^* {\bf x}^*{\bf P}^{-1}{\bf z}$ and substituting into (\ref{eq.2}), one obtains
	\begin{align}
	-(\lambda^*+\lambda){\bf x}^*{\bf P}^{-1}{\bf z} = {\bf x}^*{\bf P}^{-1}{\bf x} \label{eq.3}
	\end{align}
	Since $\lambda$ is pure imaginary, we know that the LHS of (\ref{eq.3}) is equal to zero. 
	However, since ${\bf x}\neq {\bf 0}$, the RHS of (\ref{eq.3}) is not zero. 
	Therefore, we have a contradiction concluding the proof.
\end{proof}


\begin{proof}
	\textbf{(Lemma~\ref{lem:multiplicity})\;}
	Since ${\bf J}^\top$ is skew-symmetric and ${\bf K}^\top{\bf D}({\bf y}) {\bf K} $ is positive definite, in view of Lemma~\ref{lem:diag}, we know that  $\frac{\partial \bf f}{\partial {\bf y}}$  is diagonalizable for all $\bf y$.
	Therefore,   the algebraic and geometric multiplicity of repeated eigenvalues of   $\frac{\partial \bf f}{\partial {\bf y}}$ coincides.
\end{proof}

\section{Implementation details}\label{ap:implementation}
DNN architectures and training algorithms are implemented using PyTorch library\footnote{\url{https://pytorch.org/}}. 

\subsection{Binary classification datasets}\label{ap:implementation_2d}
For two-class classification problems we used 5000 datapoints and a mini-batch size of 125, for both training and test data. For the optimization algorithm, we use coordinate gradient descent, i.e. a modified version of stochastic gradient descent (SGD) with Adam ($\beta_1 = 0.9$, $\beta_2 = 0.999$) \cite{Haber_2017}. The cross-entropy loss has been used to compare the predicted outputs and the true labels.
In every iteration of the algorithm, first the optimal weights of the output layer are computed given the last updated parameters of the hidden layers, and then, a step update of the hidden layers' parameters is performed by keeping fixed the output parameters. 
The training consist in 50 epochs and each of them has maximum 10 iterations to compute the output layer weights. 
The learning rate was set to 0.05 and the weight decay for the output layer is constant and set to $1 \times 10^{-4}$.
The $\alpha$ coefficient of the regularization is set to $5 \times 10^{-3}$.

\subsection{MNIST dataset}\label{ap:implementation_mnist}
We use the complete MNIST dataset (60,000 training examples and 10,000 test examples) and a mini-batch size of 100. 
For the optimization algorithm we use SGD with Adam and cross-entropy loss.
The learning rate is initialized to be 0.04 and decayed with $\gamma = 0.8$ each epoch. The total training step is 40 epochs. 
We use $\alpha = 1\times 10^{-3}$ for the regularization term in all H-DNNs.
The weight decay is constant and set to $2\times 10^{-4}$.

\subsection{Fully-connected neural network}\label{ap:implementation_fc}
The layer equation of a fully connected network is ${\bf y}_{k+1}=\sigma({\bf K}_k{\bf y}_{k}+{\bf b}_k)$, with activation function $\sigma(\cdot) = \tanh(\cdot)$ and trainable parameters ${\bf K}_k \in \mathbb{R}^{n \times n}$ and ${\bf b}_k \in \mathbb{R}^n$ for $k=0,1,\dots,N-1$ where $N$ is the number of layers of the network. We use the same implementation as described in Appendix~\ref{ap:implementation_2d}, with a weight decay of $2\times 10^{-4}$.

\end{document}